\documentclass[10pt]{article} 
\usepackage[preprint]{rlc}

\usepackage{amssymb}            
\usepackage{mathtools}          
\usepackage{mathrsfs}           
\mathtoolsset{showonlyrefs}     
\usepackage{graphicx}           
\usepackage{subcaption}         
\usepackage[space]{grffile}     
\usepackage{url}                

\usepackage{graphicx}
\usepackage{booktabs}
\usepackage{amsmath}
\usepackage{amsthm}
\usepackage{amsfonts}
\usepackage{bm}
\usepackage{bbm}
\usepackage{thmtools} 
\usepackage{thm-restate}
\newtheorem{lemma}{Lemma}

\newtheorem{theorem}{Theorem}

\usepackage{algorithm}
\usepackage{caption}
\usepackage{subcaption}

\usepackage{algpseudocode}
\algnewcommand\algorithmicforeach{\textbf{for each}}
\algdef{S}[FOR]{ForEach}[1]{\algorithmicforeach\ #1\ \algorithmicdo}

\title{LinearAPT: An Adaptive Algorithm for the Fixed-Budget Thresholding Linear Bandit Problem}

\author{Yun-Ang Wu  \\
    r11944072@ntu.edu.tw \\
    National Taiwan University
    \And
    Yun-Da Tsai \\
    f08946007@csie.ntu.edu.tw\\
    National Taiwan University
    \And
    Shou-De Lin \\
    sdlin@csie.ntu.edu.tw\\
    National Taiwan University
}

\begin{document}

\maketitle

\begin{abstract}
    In this study, we delve into the Thresholding Linear Bandit (TLB) problem, a nuanced domain within stochastic Multi-Armed Bandit (MAB) problems, focusing on maximizing decision accuracy against a linearly defined threshold under resource constraints. We present LinearAPT, a novel algorithm designed for the fixed budget setting of TLB, providing an efficient solution to optimize sequential decision-making. This algorithm not only offers a theoretical upper bound for estimated loss but also showcases robust performance on both synthetic and real-world datasets. Our contributions highlight the adaptability, simplicity, and computational efficiency of LinearAPT, making it a valuable addition to the toolkit for addressing complex sequential decision-making challenges.
\end{abstract}

\section{Introduction}

Stochastic Multi-Armed Bandit (MAB) problems have long been pivotal in machine learning, aimed at optimizing sequential decision-making to maximize cumulative rewards. This domain has expanded to include the Thresholding Bandit Problem (TBP), which focuses on identifying options surpassing a certain performance threshold, catering to a wide array of applications such as binary classification and anomaly detection, and highlighting the importance of decision-making in uncertain environments.

Our research focus on the less explored Thresholding Linear Bandit (TLB) problem, where decisions are assessed against a specific threshold, and the mean rewards of arms adhere to a linear model with undetermined parameters. The study of TLB unfolds under two principal scenarios: the fixed budget setting, where the objective is to identify actions above the threshold within a limited number of trials, and the fixed confidence setting, which aims at minimizing trial numbers to achieve a certain confidence level in decision accuracy.

In addressing the fixed budget scenario of TLB, we introduce a novel algorithm specifically crafted for this context. This method underlines efficient exploration and exploitation within set trial limits, strategically optimizing decision-making under resource constraints. By focusing on the fixed budget scenario, our contribution presents a pragmatic approach to the TLB dilemma, enriching the arsenal for navigating complex sequential decision-making challenges.


Our contributions can be summarized as follows:
\begin{itemize}
    \item We introduce an algorithm, LinearAPT, tailored for the novel fixed-budget thresholding bandit setting, accompanied by a theoretical upper bound for estimated loss.
    \item Our algorithm demonstrates competitive performance across synthetic and real-world datasets. It exhibits adaptability, simplicity of implementation, and computational efficiency, bolstering its utility and effectiveness.
\end{itemize}



\section{Related Works}
\label{sec:related}

Table~\ref{tab:problem-setting} succinctly summarizes the landscape of research on various TBP variants, highlighting our contribution as the first to provide an upper bound for the fixed budget linear thresholding bandit problem.

\subsection{Unstructured Thresholding Bandit Problems}
Both the fixed budget and fixed confidence settings of the unstructured TBP have been extensively explored. Works such as those by \cite{NIPS2014_e56954b4} and \cite{pmlr-v48-locatelli16} provide foundational studies in these areas. Additionally, variations of the unstructured TBP, including studies on aggregate regret \cite{NEURIPS2019_9a0684d9} and shape-constrained models beneficial for applications like medical trials \cite{garivier2018thresholding,pmlr-v125-cheshire20a,pmlr-v139-cheshire21a}, have expanded the understanding of TBP under specific constraints, such as monotonically increasing mean rewards.

\subsection{Structured Thresholding Bandit Problems}
The structured TBP, including linear models and graph bandit cases, has seen significant developments. Linear pure exploration \cite{pmlr-v119-degenne20a} and graph bandit TBP \cite{pmlr-v108-lejeune20a} under both fixed budget and fixed confidence settings have been studied, offering algorithms that are asymptotically optimal in certain scenarios. Notably, our work is the first to address linear TBP in the fixed budget setting, filling a gap in the current literature.


\subsection{Level Set Estimation}
The study of level set estimation \cite{10.5555/2540128.2540322} presents another angle on structured TBP, focusing on classifying elements of a domain based on their relation to a threshold. This approach models the underlying function as a sample from a Gaussian process, offering insights into structured TBP where the function exhibits specific characteristics, such as linearity.

\begin{table}[t]
\centering
    \bgroup
    \def\arraystretch{1.5}
    \begin{tabular}{ c||c|c } 
         & Fixed Budget & Fixed confidence  \\ 
        \hline\hline
        Unstructured & \cite{pmlr-v48-locatelli16} & \cite{NIPS2014_e56954b4} \\ 
        \hline
        Linear & Our Paper & \cite{pmlr-v119-degenne20a} \\ 
    \end{tabular}
    \egroup
\caption{A list of papers presenting results on various variants of the thresholding bandit problem. As per our current knowledge, this paper is the first to provide an upper bound for the fixed budget linear thresholding bandit problem.}
\label{tab:problem-setting}
\end{table}


\section{Problem Definition}

\begin{table}[htbp]
    \begin{center}
        \begin{tabular}{r c p{10cm} }
        \toprule
        $K$ & Number of arms \\
        $\mathcal{A}$ &  set of arms \\
        $d$ &  Dimension of arm vectors \\
        $\theta$ &  Regression parameter \\
        $\hat{\theta}_t$ &  Estimated regression parameter at round $t$\\
        $L$ &  Upper bound of arm vector \\
        $\epsilon$ &  Precision \\
        $\tau$ &  Threshold \\
        $T$ & Budget \\
        $x_i$ &  Arm vector of arm $i$ \\
        $x_{\phi(t)}$ &  Arm selected at round $t$ \\
        $r_t$ &  Returned reward at round $t$ \\
        $\mu_i$ &  Mean reward of arm $i$ \\
        $\hat{\mu}_i(t)$ &  Estimated reward of arm $i$ at round $t$ \\
        $\Delta_i$ & Gap for arm $i$, defined as $\lvert \mu_i - \tau \rvert + \epsilon$   \\
        $\hat{\Delta}_i(t)$ & Estimated gap for arm $i$ at round $t$, defined as $\lvert \hat{\mu}_i(t) - \tau \rvert + \epsilon$   \\
        $H$ & Complexity, defined as $\sum_{i=1}^K {\Delta_i^{-2}}$ \\
        $T_i(t)$ & Number of pulls up to time $t$ for arm $i$ \\
        $B_i(t)$ & $\sqrt{T_{i}(t-1)}\hat{\Delta}_i(t-1)$ \\
        
        \bottomrule
        \end{tabular}
    \caption{Table of Notations used}
    \end{center}
    \label{tab:TableOfNotation}
\end{table}

\subsection{Linear Bandit}

Consider $\mathcal{A}$, a finite subset within the $d$-dimensional real space $\mathbb{R}^d$, representing the set of available actions or "arms." Let $K = |\mathcal{A}|$ denote the total number of arms, with ${x_1, ..., x_K} = \mathcal{A}$. The terminology "arm $x_i$" and "arm $i$" are used interchangeably for simplicity. The expected reward for choosing arm $x_i$ is given by $\mu_i = \langle \theta, x_i \rangle$, which is the inner product of the arm's feature vector $x_i$ and an unknown parameter vector $\theta$.

We introduce the notion of an {\it R-sub-Gaussian} distribution, a common assumption in the analysis of bandit algorithms. Formally, a distribution $\nu$ is considered R-sub-Gaussian if for any real number $t$,
$$
\forall t\in\mathbb{R}, \quad  \mathbb{E}_{X \sim \nu}[\exp (tX - t\mathbb{E}[X])] \leq \exp \left(\frac{R^2t^2}{2}\right)
$$
This condition ensures a controlled behavior of reward distributions, allowing us to derive theoretical guarantees. For the sake of clarity and without loss of generality, we set the variance parameter $R = 1$ throughout our analysis. Upon selecting an arm, the algorithm receives a reward sampled from an R-sub-Gaussian distribution $\nu_i$ with mean $\mu_i$.
Additionally, we assume that the Euclidean norm of each arm is bounded, specifically, $\forall i \in [K], \lVert x_i \rVert \leq L$. This constraint on the arm vectors is crucial for ensuring the feasibility of the learning problem within our defined space.

\subsection{Thresholding Bandit}

This subsection elucidates the objective within the thresholding bandit problem framework, adhering to the notation conventions established in \cite{pmlr-v48-locatelli16}. The primary goal revolves around categorizing all arm means relative to a specified threshold value. Let $\tau \in \mathbb{R}$ represent the threshold, and let $\epsilon \geq 0$ denote the precision level. We define $S_u$ as the subset of arms from $\mathcal{A}$ whose mean values exceed $u$, i.e., $S_u = \{ x_i \in \mathcal{A} | \mu_i \geq u\}$. Conversely, $S_u^C$ comprises those arms with mean values below $u$. The task is to accurately identify arms within $S_{\tau + \epsilon}$ and $S_{\tau - \epsilon}^C$, thus classifying arms with a precision up to $\epsilon$.

To estimate the set $S_u$ at round $t$, we introduce $\hat{S}_u(t) = \{ x_i \in \mathcal{A} | \hat{\mu}_i(t) \geq u\}$. The effectiveness of an algorithm after $T$ rounds is quantified by its loss, expressed as
$$
\mathcal{L} (T) = \mathbb{I}(S_{\tau + \epsilon} \cap \hat{S}_\tau(T)^C \neq \emptyset \vee S_{\tau - \epsilon}^C \cap \hat{S}_\tau(T) \neq \emptyset)
$$
which underscores the necessity for precise classification of arms whose mean rewards are either above $\tau + \epsilon$ or below $\tau - \epsilon$. The inherent randomness in arm sampling influences this loss metric. Consequently, minimizing the expected loss, defined as
$$
\mathbb{E}[\mathcal{L} (T)] = \mathbb{P}(S_{\tau + \epsilon} \cap \hat{S}_\tau(T)^C \neq \emptyset \vee S_{\tau - \epsilon}^C \cap \hat{S}_\tau(T) \neq \emptyset)
$$
becomes the algorithm's overarching objective.

Furthermore, acknowledging the varying complexities inherent to different scenarios is critical. Hence, we introduce a measure of problem complexity, inspired by \cite{pmlr-v48-locatelli16}, which is formulated as
$$
H = \sum_{i=1}^K \Delta_i^{-2} = \sum_{i=1}^K (\lvert\mu_i - \tau \rvert + \epsilon)^{-2}
$$
highlighting how $H$ is contingent upon both the threshold $\tau$ and the precision $\epsilon$. This complexity metric serves as a vital parameter in understanding and tackling the thresholding bandit problem.

\section{Algorithm}

\begin{algorithm}[t]
\label{alg:cap}
    \begin{algorithmic}
    \Require $n \geq 0$
    \State $V_0 \gets I$
    \State $b_0 \gets 0$
    \State $\hat{\theta}_0 \gets 0$ 
    \For{$t = 1,...,T$}
        \If {$t \in [K]$} 
            \State Pull arm $t$
        \Else
            \State Pull arm $x_{\phi(t)} = \arg\min_{k \leq K} B_k(t)$
        \EndIf
            \State Observe reward $r_t \sim \nu_{\phi(t)}$
            \State $V_t \gets V_{t-1} + x_{\phi(t)} x_{\phi(t)}^T$
            \State $b_t \gets b_{t-1} + r_t x_{\phi(t)}$
            \State $\hat{\theta}_t \gets V_{t}^{-1}b_t$
    \EndFor
    \end{algorithmic}
    \caption{LinearAPT}
\end{algorithm}

\subsection{LinearAPT}

We propose our algorithm, LinearAPT, which is derived from the APT (Anytime Parameter-free Thresholding) algorithm~\cite{pmlr-v48-locatelli16}. We also provide an upper bound on the expected loss, which will be presented in the next section. Unlike the scenario with unstructured bandit problems where the estimated mean, $\hat{\mu}_s$, is calculated as the sum of all previous rewards, our context requires a more nuanced approach due to the linear relationship between the reward and the arm vector. To this end, after $s$ rounds, given the pairings of arm vectors and rewards, $(x_{\phi(t)}, r_t)_{t \leq s}$, we aim to estimate the hidden parameter $\theta$.

This estimation is achieved by addressing the optimization problem:
$$
\hat{\theta}_s = {\arg\min}_{\theta \in \mathbb{R}^d} \left( \sum_{t=1}^s \left(r_t - \langle \theta, x_{\phi(t)} \rangle\right)^2 + \lambda \lVert \theta \rVert^2 \right)
$$

which essentially performs a regularized regression on the available data. The regularization term, $\lambda \lVert \theta \rVert^2$, ensures uniqueness of the solution by preventing ill-conditioned outcomes when the set ${x_{\phi(1)}, ..., x_{\phi(s)}}$ does not span the entire space $\mathbb{R}^d$.

The solution to this optimization problem can be efficiently computed in a recursive manner using the following formulations:
$$
b_0 = 0, \quad b_t = b_{t-1} + r_t x_{\phi(t)}
\quad \text{where} \quad \forall t \enspace b_t \in \mathbb{R}^d
$$
$$
V_0 = \lambda I, \quad V_t = V_{t-1} + x_{\phi(t)} x_{\phi(t)}^T 
\quad \text{where} \quad \forall t \enspace V_t \in \mathbb{R}^{d\times d}
$$

The recursive solution yields $\hat{\theta}_s = V_s^{-1} b_s$ as the optimal estimate for the hidden parameter $\theta$, demonstrating the efficiency and effectiveness of LinearAPT in leveraging linear relationships in bandit problems.


\subsection{Theoretical Bound}

Before presenting our principal theorem, we introduce several lemmas that lay the foundation for our analysis:

\begin{lemma} (\cite{NIPS2014_f387624d}) For any $\delta > 0$, with probability at least $1-\delta$:
$$
\forall t \geq 0, \forall x \in \mathbb{R}^d \quad |x^T\theta - x^T\hat{\theta}_t| \leq \lVert x\rVert_{V_t^{-1}}\left(\sqrt{d\log\left(\frac{1+tL^2}{\delta}\right)} + \lVert\theta\rVert\right)
$$
\end{lemma}
\begin{proof}
This is obtained by setting $\lambda = 1$ in Proposition 2 of \cite{NIPS2014_f387624d}.
\end{proof}

\begin{lemma}
(\cite{pmlr-v130-reda21a}) For any $t > 0$, for any $i \in [K]$ such that $T_i(t) > 0$, for all $x \in \mathbb{R}^n$:
$$
\lVert x\rVert^2_{(V_t)^{-1}} \leq x^T(I + T_i(t)x_ix_i^T)^{-1}x
$$
\end{lemma}
\begin{proof}
This is concluded by setting $\lambda = 1$ in Lemma 11 of \cite{pmlr-v130-reda21a}.
\end{proof}

\begin{lemma}
For any $t > 0$ and for any arm $i \in [K]$ with $T_i(t) > 0$, we have
$$
\lVert x_i\rVert_{V_t^{-1}} \leq \frac{1}{\sqrt{T_i(t)}}
$$
\end{lemma}
\begin{proof}
    See Appendix~\ref{app:lemma3}.
\end{proof}

With these preliminaries established, we present our main result as follows:
\begin{theorem}
Assume all arm distribution is $1$-sub-gaussian and $\lVert x_i\rVert \leq L$, $\lVert\theta\rVert < \sqrt{\frac{T}{\gamma^2 H}}$ where $\gamma = 4$. Algorithm LinearAPT's expected loss is upper bounded by
$$
\mathbb{E}[\mathcal{L}(T)] \leq \exp\left\{\log(1 + TL^2) - \frac{1}{d}\left(\sqrt{\frac{T}{\gamma^2H}} - \lVert\theta\rVert\right)^2\right\}
$$
\end{theorem}

\begin{proof}
    See Appendix~\ref{app:theorem1}.
\end{proof}

\subsection{Analysis}

As demonstrated in Table \ref{tab:lower-bound}, the upper bound derived for our approach does not depend on $K$, the number of arms. This characteristic aligns with our intuition, as our primary objective involves estimating the underlying parameter rather than the mean of each arm individually. If a reliable approximation of the original parameter is obtained, it enables immediate estimation for any additional arms introduced subsequently. This contrasts sharply with the unstructured scenario, where the absence of auxiliary information necessitates a bound that escalates with $K$.

However, our upper bound incorporates a term $-1/d$, which emerges from the $\sqrt{d}$ component found in Proposition 2 of \cite{DBLP:journals/corr/SoareLM14}. Previous studies have illuminated that this term represents a compromise for achieving adaptability within the proposition. The essence of this trade-off lies in the observation that as the dimensionality increases, the estimated parameter encapsulates diminishing informational value, necessitating a larger budget to ensure the precision of the estimation.

Moreover, our proposed algorithm exhibits computational efficiency. The primary computational demand arises from the inversion of a $d \times d$ matrix required during the update of the estimated parameter $\hat{\theta}$. Consequently, the overall computational complexity of our algorithm is $\mathcal{O}(Td^\phi)$, where $\phi$ is contingent upon the specific matrix inversion technique employed. This feature underscores the practicality of our algorithm, especially in contexts where computational resources are a consideration.

\begin{center}
\begin{table}[t]
\centering
    \bgroup
    \def\arraystretch{1.5}
    \begin{tabular}{ c|c } 
        Scenario & Upper bound \\ 
        \hline
        Unstructured \cite{pmlr-v48-locatelli16} &  $ \mathbb{E}[\mathcal{L}(T)] \asymp \left( -\frac{T}{H} + \log (\log (T)K )\right)$ \\ 
        \hline
        Graph  \cite{pmlr-v108-lejeune20a} & $ \mathbb{E}[\mathcal{L}(T)] \asymp \left( -\frac{T}{H} +  d_T \log (T)\right)$ \\
        \hline
        Linear & $ \mathbb{E}[\mathcal{L}(T)] \asymp \left( -\frac{T}{dH} + \log (T)\right)$  \\ 
    \end{tabular}
    \egroup
\caption{A comparison on theoretical upper bound with other paper solving different variants of thresholding bandit problem.}
\label{tab:lower-bound}
\end{table}
\end{center}

\section{Experiments}

\subsection{Setup}
To empirically estimate $\mathbb{E}[\mathcal{L}(T)]$, the expected loss over time $T$, we employ an experimental approximation method defined as follows:
$$
\mathbb{E}[\mathcal{L}(T)] \approx \frac{1}{N}\sum_{i = 1}^{N} \mathcal{L}_i(T)
$$
where $\mathcal{L}_i(T)$ denotes the loss observed in the $i$th simulation. In our experiments, we have chosen $N = 10,000$ to ensure statistical significance. The failure rate obtained through this method serves as an estimate of our algorithm's performance.

\subsection{Dataset}

Our evaluation compares the effectiveness of our algorithm against established baseline methods across two synthetic datasets and two real-world datasets. This approach allows us to rigorously assess our algorithm's adaptability and performance in both controlled and practical scenarios.

\paragraph{Synthetic Dataset} All arm vector $\{x_1, ..., x_K\} = \mathcal{X}$ is sampled from the uniform distribution $\mathcal{U}\left([-1,1]^{d}\right)$. The parameter $\theta$ is also random sampled from the same distribution. We set $K=20, \tau=0$ and $\epsilon = 0.01$. We measure the performance of algorithms at $T = 40, 80, 120, 160, 200$. To understand the effect of dimension on the algorithm's performance, we perform experiments on $d=5, 20$ respectively.

\paragraph{Real-world Dataset} For our experiment, we use a modified version of iris and wine dataset from sklearn. 
While these two datasets were initially designed for classification tasks, we can transform them into linear regression datasets by sampling a random linear function. This approach follows the methodology commonly used in the literature, as seen in ~\cite{dudik2011doubly,da2022fast,tsai2023differential,tsai2024lil}. For example, a data instance in iris has the shape of $(x, y)\in \mathbb{R}^{4} \times \{1,2,3\}$ where $x$ is the feature and $y$ is the label. A synthetic parameter $\theta$ is randomly sampled from $\mathcal{U}\left([-1,1]^{4}\right)$, the pseudo reward for $x$ is then defined as $\langle x, \theta\rangle$. The point is the distribution of arm is more natural when the dataset is constructed in this way. We set the threshold as the mean value of the rewards $\frac{1}{N}\sum_{i=1}^N \langle \theta, x_i\rangle$ and the precision $\epsilon = 0.1$. The budget is $T = 200, 250, 300, 350, 400$.

\subsection{Baseline}

We first introduce two baselines, Random and APT. As explained in the following section, linearAPT can be viewed as an composition of these two algorithm. This makes important to compare linearAPT with these two algorithm, as it serves like an ablation study for linearAPT.
\paragraph{Random} This algorithm has a random arm selection rule, that is, the next exampled arm is determined by a random sample from $[K]$. the algorithm also exploit the knowledge of the underlying linear structure by estimating the hidden parameter $\theta$ by solving the optimization problem mentioned earlier. In other words, the only difference between Random and LinearAPT is the selection rule in line 8.

\paragraph{APT} APT is the algorithm introduced by \cite{pmlr-v48-locatelli16}. Since this algorithm solves the unstructured version of TBP, it is easy to assume that linearAPT performs better under datasets with linear structure. linearAPT coincides with APT when the estimated mean is calculated by the average of the reward history of the target arm, not by the inner product between the arm vector and the estimated parameter.

\paragraph{UCBE} We also follow \cite{pmlr-v48-locatelli16} and include UCB-type algorithm as our baseline. The UCBE algorithm was introduced as a baseline in \cite{pmlr-v48-locatelli16}. The selection rule in line 8 of our algorithm was changed into $\arg\min_{k} \hat\Delta_k - \sqrt{\frac{a}{T_k(t)}}$. similar to \cite{pmlr-v48-locatelli16}, we set $a = 4^i\frac{T-K}{H}$ for $i = \{-1, 0, 4\}$. Note that in a realistic scenario, it is not possible to calculate $a$ using $i$ since we do not have access to the complexity $H$.


\begin{figure}[h]
  \begin{minipage}{0.5\textwidth}
    \includegraphics[width=\textwidth]{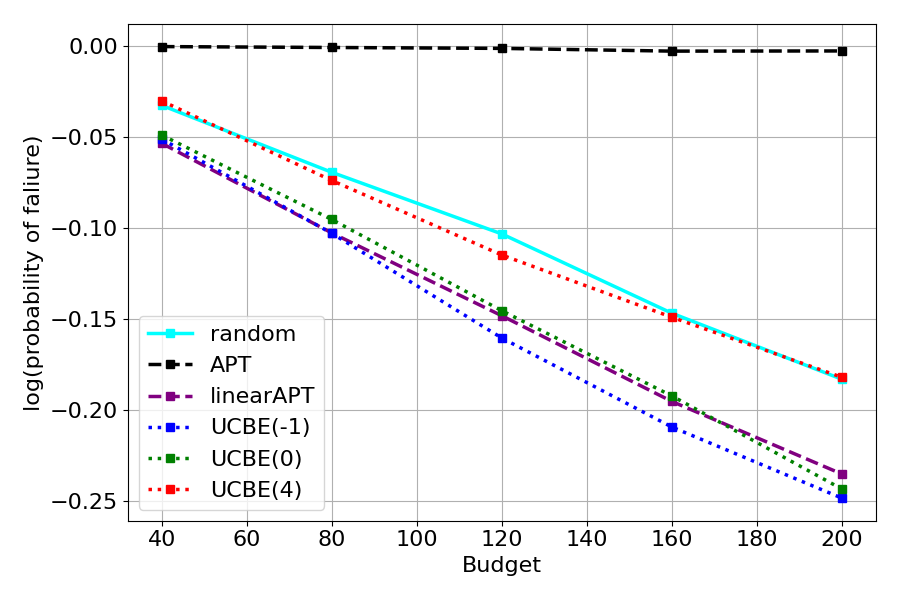}
    \captionsetup{labelformat=empty}
    \caption{(a) Uniform Box, $d = 5$}
  \end{minipage}
  \hfill
  \begin{minipage}{0.48\textwidth}
    \includegraphics[width=\textwidth]{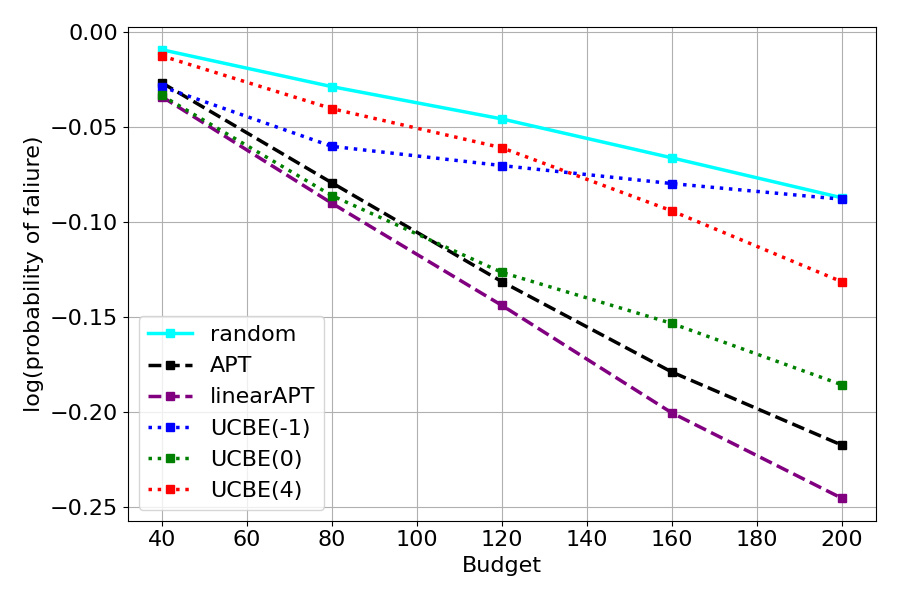}
    \captionsetup{labelformat=empty}
    \caption{(b) Uniform Box, $d = 20$}
  \end{minipage}
  \caption{LinearAPT demonstrates competitive performance in the experiment results on synthetic data. The y-axis represents the logarithm of the probability of failure. The left and right figures represent the experimental results for $d=5$ and $d=20$ respectively. When $d$ is large, the performance of unstructured algorithms improves.}
  \label{fig:synthetic}
\end{figure}

\begin{figure}[h]
  \begin{minipage}{0.5\textwidth}
    \includegraphics[width=\textwidth]{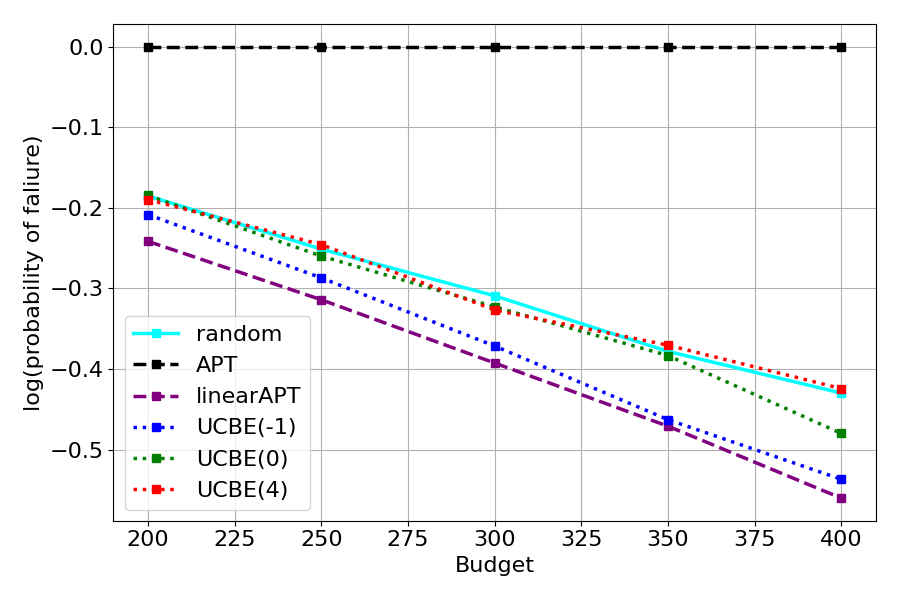}
    \captionsetup{labelformat=empty}
    \caption{(a) Modified version of iris dataset}
  \end{minipage}
  \hfill
  \begin{minipage}{0.48\textwidth}
    \includegraphics[width=\textwidth]{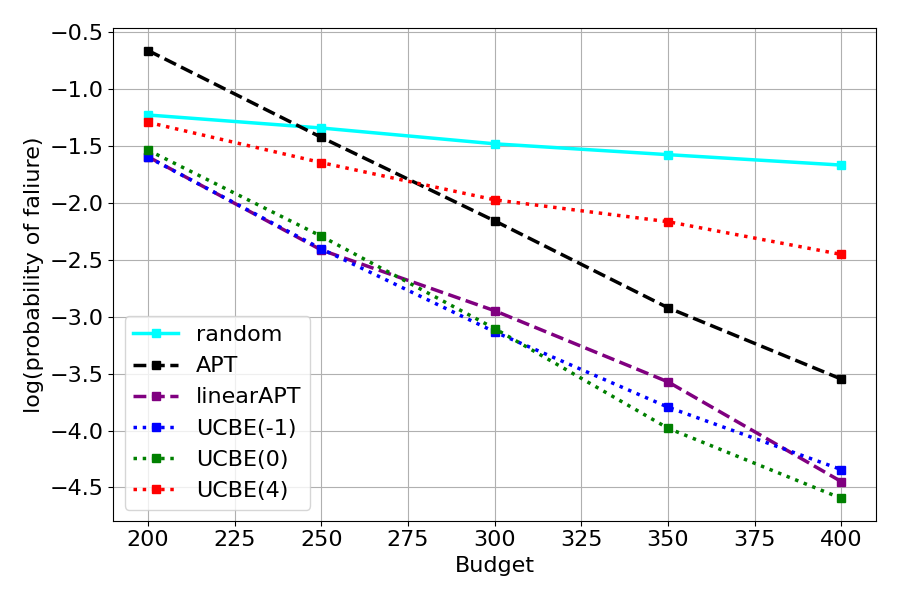}
    \captionsetup{labelformat=empty}
    \caption{(b) Modified version of wine dataset}
  \end{minipage}
  \label{fig:realworld}
  \caption{LinearAPT demonstrates competitive performance on real-world data. The left and right figures represent the experimental results for the iris and wine datasets, respectively. Similar behavior to the synthetic data can be observed.}
\end{figure}

\subsection{Results}

The performance of LinearAPT across various settings, as depicted in the figures, showcases its competitiveness. To ensure a thorough evaluation, the budget for each dataset was intentionally set to exceed the total number of arms, accommodating our algorithm's prerequisite for sampling each arm at least once. Specifically, for a synthetic dataset with $K=20$ arms, the budget commenced at $T=40$. Similarly, for the Iris and Wine datasets, featuring $K=150$ and $178$ arms respectively, the starting budget was established at $T=200$.

With carefully tuned hyperparameter $i$, UCBE can compete with our algorithm's performance. In 3 out of 4 of our experiments, UCBE(-1) can compete with LinearAPT, and in 3 out of 4 of our experiments, UCBE(0) can compete with LinearAPT. The problem is when facing an unseen problem, we usually do not know which hyperparameter $i$ is better. Secondly, even if we have knowledge about which $i$ is better, we cannot calculate $a = 4^i\frac{T-K}{H}$ without access to complexity $H$. We can view LinearAPT as a one-for-all universal solution to this problem.

In the synthetic dataset scenario, illustrated in \ref{fig:synthetic}, we observe that as the dimensionality increases from $5$ to $20$, the APT algorithm begins to rival our LinearAPT approach. This aligns with the earlier discussion that higher dimensions necessitate a greater number of samples before the estimation of the parameter becomes sufficiently accurate. The presence of a $-d/T$ term in our upper bound suggests that LinearAPT excels when this term is significant. This pattern is also evident in real-world datasets, where the dimensionality of the action space is $4$ for the Iris dataset and $13$ for the Wine dataset, indicating a similar trend in performance across varying dimensions.
 
Additionally, an interesting observation arises when comparing the random algorithm to UCBE in low-dimensional settings ($d=5$). In such cases, any arbitrary direction contributes to refining the estimation of the regression parameter, rendering the random algorithm surprisingly competitive. This highlights the nuanced interplay between dimensionality and algorithm performance, underscoring the adaptability of LinearAPT in both synthetic and real-world environments.


\section{Conclusion}

In this paper, we introduce LinearAPT, a theoretically justified algorithm that addresses a novel variant of a classic bandit pure exploration problem. The algorithm demonstrates competitive performance on both synthetic and real-world datasets and is computationally efficient. Its computational complexity and theoretical upper bound scale well with the dataset size, making it a practical choice for handling large datasets.

The theoretical lower bound of the fixed confidence TLB problem was established by previous work \cite{pmlr-v119-degenne20a}. However, as of our current knowledge, the lower bound of the fixed-budget TLB problem remains unknown. We also lack knowledge on whether the lower bound for the fixed confidence case can be transformed into our setting. Further investigation has been left for future work.

For the algorithm design and proof, we draw significant inspiration from \cite{pmlr-v48-locatelli16} and \cite{pmlr-v108-lejeune20a}. Exploring whether similar techniques generalize to other structured bandit settings represents a promising avenue for future research.

%

\bibliography{main}

\begin{thebibliography}{16}
\providecommand{\natexlab}[1]{#1}
\providecommand{\url}[1]{\texttt{#1}}
\expandafter\ifx\csname urlstyle\endcsname\relax
  \providecommand{\doi}[1]{doi: #1}\else
  \providecommand{\doi}{doi: \begingroup \urlstyle{rm}\Url}\fi

\bibitem[Chen et~al.(2014)Chen, Lin, King, Lyu, and Chen]{NIPS2014_e56954b4}
Shouyuan Chen, Tian Lin, Irwin King, Michael~R Lyu, and Wei Chen.
\newblock Combinatorial pure exploration of multi-armed bandits.
\newblock In Z.~Ghahramani, M.~Welling, C.~Cortes, N.~Lawrence, and K.Q. Weinberger (eds.), \emph{Advances in Neural Information Processing Systems}, volume~27. Curran Associates, Inc., 2014.
\newblock URL \url{https://proceedings.neurips.cc/paper_files/paper/2014/file/e56954b4f6347e897f954495eab16a88-Paper.pdf}.

\bibitem[Cheshire et~al.(2020)Cheshire, Menard, and Carpentier]{pmlr-v125-cheshire20a}
James Cheshire, Pierre Menard, and Alexandra Carpentier.
\newblock The influence of shape constraints on the thresholding bandit problem.
\newblock In Jacob Abernethy and Shivani Agarwal (eds.), \emph{Proceedings of Thirty Third Conference on Learning Theory}, volume 125 of \emph{Proceedings of Machine Learning Research}, pp.\  1228--1275. PMLR, 09--12 Jul 2020.
\newblock URL \url{https://proceedings.mlr.press/v125/cheshire20a.html}.

\bibitem[Cheshire et~al.(2021)Cheshire, Menard, and Carpentier]{pmlr-v139-cheshire21a}
James Cheshire, Pierre Menard, and Alexandra Carpentier.
\newblock Problem dependent view on structured thresholding bandit problems.
\newblock In Marina Meila and Tong Zhang (eds.), \emph{Proceedings of the 38th International Conference on Machine Learning}, volume 139 of \emph{Proceedings of Machine Learning Research}, pp.\  1846--1854. PMLR, 18--24 Jul 2021.
\newblock URL \url{https://proceedings.mlr.press/v139/cheshire21a.html}.

\bibitem[Degenne et~al.(2020)Degenne, Menard, Shang, and Valko]{pmlr-v119-degenne20a}
R{\'e}my Degenne, Pierre Menard, Xuedong Shang, and Michal Valko.
\newblock Gamification of pure exploration for linear bandits.
\newblock In Hal~Daumé III and Aarti Singh (eds.), \emph{Proceedings of the 37th International Conference on Machine Learning}, volume 119 of \emph{Proceedings of Machine Learning Research}, pp.\  2432--2442. PMLR, 13--18 Jul 2020.
\newblock URL \url{https://proceedings.mlr.press/v119/degenne20a.html}.

\bibitem[Dud{\'\i}k et~al.(2011)Dud{\'\i}k, Langford, and Li]{dudik2011doubly}
Miroslav Dud{\'\i}k, John Langford, and Lihong Li.
\newblock Doubly robust policy evaluation and learning.
\newblock \emph{arXiv preprint arXiv:1103.4601}, 2011.

\bibitem[Garivier et~al.(2018)Garivier, Ménard, Rossi, and Menard]{garivier2018thresholding}
Aurélien Garivier, Pierre Ménard, Laurent Rossi, and Pierre Menard.
\newblock Thresholding bandit for dose-ranging: The impact of monotonicity, 2018.

\bibitem[Gotovos et~al.(2013)Gotovos, Casati, Hitz, and Krause]{10.5555/2540128.2540322}
Alkis Gotovos, Nathalie Casati, Gregory Hitz, and Andreas Krause.
\newblock Active learning for level set estimation.
\newblock In \emph{Proceedings of the Twenty-Third International Joint Conference on Artificial Intelligence}, IJCAI '13, pp.\  1344–1350. AAAI Press, 2013.
\newblock ISBN 9781577356332.

\bibitem[LeJeune et~al.(2020)LeJeune, Dasarathy, and Baraniuk]{pmlr-v108-lejeune20a}
Daniel LeJeune, Gautam Dasarathy, and Richard Baraniuk.
\newblock Thresholding graph bandits with grapl.
\newblock In Silvia Chiappa and Roberto Calandra (eds.), \emph{Proceedings of the Twenty Third International Conference on Artificial Intelligence and Statistics}, volume 108 of \emph{Proceedings of Machine Learning Research}, pp.\  2476--2485. PMLR, 26--28 Aug 2020.
\newblock URL \url{https://proceedings.mlr.press/v108/lejeune20a.html}.

\bibitem[Locatelli et~al.(2016)Locatelli, Gutzeit, and Carpentier]{pmlr-v48-locatelli16}
Andrea Locatelli, Maurilio Gutzeit, and Alexandra Carpentier.
\newblock An optimal algorithm for the thresholding bandit problem.
\newblock In Maria~Florina Balcan and Kilian~Q. Weinberger (eds.), \emph{Proceedings of The 33rd International Conference on Machine Learning}, volume~48 of \emph{Proceedings of Machine Learning Research}, pp.\  1690--1698, New York, New York, USA, 20--22 Jun 2016. PMLR.
\newblock URL \url{https://proceedings.mlr.press/v48/locatelli16.html}.

\bibitem[R{\'e}da et~al.(2021)R{\'e}da, Kaufmann, and Delahaye-Duriez]{pmlr-v130-reda21a}
Cl{\'e}mence R{\'e}da, Emilie Kaufmann, and Andr{\'e}e Delahaye-Duriez.
\newblock Top-m identification for linear bandits.
\newblock In Arindam Banerjee and Kenji Fukumizu (eds.), \emph{Proceedings of The 24th International Conference on Artificial Intelligence and Statistics}, volume 130 of \emph{Proceedings of Machine Learning Research}, pp.\  1108--1116. PMLR, 13--15 Apr 2021.
\newblock URL \url{https://proceedings.mlr.press/v130/reda21a.html}.

\bibitem[Soare et~al.(2014{\natexlab{a}})Soare, Lazaric, and Munos]{DBLP:journals/corr/SoareLM14}
Marta Soare, Alessandro Lazaric, and R{\'{e}}mi Munos.
\newblock Best-arm identification in linear bandits.
\newblock \emph{CoRR}, abs/1409.6110, 2014{\natexlab{a}}.
\newblock URL \url{http://arxiv.org/abs/1409.6110}.

\bibitem[Soare et~al.(2014{\natexlab{b}})Soare, Lazaric, and Munos]{NIPS2014_f387624d}
Marta Soare, Alessandro Lazaric, and Remi Munos.
\newblock Best-arm identification in linear bandits.
\newblock In Z.~Ghahramani, M.~Welling, C.~Cortes, N.~Lawrence, and K.Q. Weinberger (eds.), \emph{Advances in Neural Information Processing Systems}, volume~27. Curran Associates, Inc., 2014{\natexlab{b}}.
\newblock URL \url{https://proceedings.neurips.cc/paper_files/paper/2014/file/f387624df552cea2f369918c5e1e12bc-Paper.pdf}.

\bibitem[Tao et~al.(2019)Tao, Blanco, Peng, and Zhou]{NEURIPS2019_9a0684d9}
Chao Tao, Sa\'{u}l Blanco, Jian Peng, and Yuan Zhou.
\newblock Thresholding bandit with optimal aggregate regret.
\newblock In H.~Wallach, H.~Larochelle, A.~Beygelzimer, F.~d\textquotesingle Alch\'{e}-Buc, E.~Fox, and R.~Garnett (eds.), \emph{Advances in Neural Information Processing Systems}, volume~32. Curran Associates, Inc., 2019.
\newblock URL \url{https://proceedings.neurips.cc/paper_files/paper/2019/file/9a0684d9dad4967ddd09594511de2c52-Paper.pdf}.

\bibitem[Tsai et~al.(2024)Tsai, Tsai, and Lin]{tsai2024lil}
Tzu-Hsien Tsai, Yun-Da Tsai, and Shou-De Lin.
\newblock lil'hdoc: An algorithm for good arm identification under small threshold gap.
\newblock \emph{arXiv preprint arXiv:2401.15879}, 2024.

\bibitem[Tsai et~al.(2022)Tsai, Lin, and Lin]{da2022fast}
Yun-Da Tsai, Shou-De Lin, and Shou-De Lin.
\newblock Fast online inference for nonlinear contextual bandit based on generative adversarial network.
\newblock \emph{arXiv preprint arXiv:2202.08867}, 2022.

\bibitem[Tsai et~al.(2023)Tsai, Tsai, and Lin]{tsai2023differential}
Yun-Da Tsai, Tzu-Hsien Tsai, and Shou-De Lin.
\newblock Differential good arm identification.
\newblock \emph{arXiv preprint arXiv:2303.07154}, 2023.

\end{thebibliography}
\bibliographystyle{rlc}


\clearpage
\appendix

\section{Proof of Theorem 1}
\label{app:theorem1}

\begin{proof} The proof of theorem 1 follows a similar structure of \cite{pmlr-v48-locatelli16} and \cite{pmlr-v108-lejeune20a}. Our proof is divided into four steps:\\
\noindent \textbf{Step 1: High probability bound} \quad Assume $\lVert\theta\rVert < \sqrt{\frac{T}{\gamma^2H}}$ and set
$$
\delta = \exp\left\{\log(1 + TL^2) - \frac{1}{d}\left(\sqrt{\frac{T}{\gamma^2H}} - \lVert\theta\rVert\right)^2\right\}
$$
Then for any arm $i$ and any $t>0$:

\begin{align*}
\lvert\mu_i - \hat{\mu}_i(t)\rvert  &= |x_i^T\theta - x_i^T\hat{\theta}_t| \\
&\leq \lVert x_i\rVert_{V_t^{-1}}\left(\sqrt{d\log\left(\frac{1+tL^2}{\delta}\right)} + \lVert\theta\rVert\right) \\
&\leq \frac{1}{\sqrt{T_i(t)}}\left(\sqrt{d\log\left(\frac{1+TL^2}{\delta}\right)} + \lVert\theta\rVert\right) \\
&\leq  \sqrt{\frac{T}{\gamma^2T_a(t)H}} 
\end{align*}

\noindent \textbf{Step 2: Characterization of a helpful arm} \quad There exists an arm $k$ at time $T$ such that $T_k(T) \geq \frac{T}{H\Delta_k^2}$. Otherwise we get
$$
T = \sum_{i=1}^K T_i(T) < \sum_{i=1}^K\frac{T}{H\Delta_i^2} = T
$$
Consider the last time $t^* \leq T$ the arm $k$ was pulled. then
$$
T_k(t^*) = T_k(T) \geq \frac{T}{H\Delta_k^2}
$$

\noindent \textbf{Step 3: Bounding other arms using the helpful arm} \quad From the triangle inequality and the definitions we have
\begin{align*}
\lvert\hat{\mu}_i(t) - \mu(i)\rvert &= \lvert(\hat{\mu}_i(t) - \tau) - (\mu_i - \tau)\rvert \\&\geq \left\lvert\lvert\hat{\mu}_i(t) - \tau\rvert - \lvert\mu_i - \tau\rvert \right\rvert\\
&= \left\lvert(\lvert\hat{\mu}_i(t) - \tau\rvert  + \epsilon) - (\lvert\mu_i - \tau\rvert + \epsilon) \right\rvert\\
&= \lvert\hat{\Delta}_i(t) - \Delta_i\rvert
\end{align*}
Combining this with the bound we get in step 1 we get
$$
\Delta_i - \sqrt{\frac{T}{\gamma^2T_i(t)H}}  \leq \hat{\Delta}_i(t) \leq \Delta_i + \sqrt{\frac{T}{\gamma^2T_i(t)H}} 
$$
By the construction of the algorithm, we know that at time $t^*$, for any arm $i \in [K]$,
$$
B_k(t^*) \leq B_i(t^*)
$$
Notice that
\begin{align*}
B_i(t) &= \hat{\Delta}_i\sqrt{T_i(t)} \\
&\leq \left(\Delta_i + \sqrt{\frac{T}{\gamma^2T_i(t)H}} \right)\sqrt{T_i(t)} \\
&\leq \Delta_i\sqrt{T_i(t)} + \frac{1}{\gamma}\sqrt{\frac{T}{H}}
\end{align*}

\noindent And
\begin{align*}
 B_k(t) &\geq \hat{\Delta}_k(t)\sqrt{T_k(t)} \\
&\geq \left (\Delta_k - \sqrt{\frac{T}{\gamma^2T_k(t)H}}  \right )\sqrt{T_k(t)}\\
& \geq \left (\Delta_k - \frac{1}{\gamma}\Delta_k  \right )\sqrt{\frac{T}{H\Delta_k^2}} \\
& = \left (\frac{\gamma - 1}{\gamma}\right )\sqrt{\frac{T}{H}} \\
\end{align*}

Combine these three inequalities, we get
$$
\left (\frac{\gamma - 1}{\gamma}\right )\sqrt{\frac{T}{H}} \leq B_k(t^*) \leq B_i(t^*) \leq \Delta_i\sqrt{T_i(t^*)} + \frac{1}{\gamma}\sqrt{\frac{T}{H}}
$$
We get a lower bound on $T_i(T)$ from this
$$
\left (\frac{\gamma - 2}{\gamma}\right)^2\frac{T}{H\Delta_i^2} \leq T_i(t^*) \leq T_i(T)
$$

\noindent \textbf{Step 4: Conclusion} 
The last inequality in step 3 yields
$$
\sqrt{\frac{T}{T_i(T)H}} \leq \Delta_i\left (\frac{\gamma}{\gamma - 2}\right )
$$

\noindent Combine this with the bound on the mean and the estimated mean yields
$$
\lvert\mu_i - \hat{\mu}_i(t)\rvert \leq  \sqrt{\frac{T}{\gamma^2T_i(t)H}}  \leq \Delta_i\left (\frac{1}{\gamma - 2}\right ) = \frac{\Delta_i}{2}
$$

For arms such that $\mu_i \geq \tau + \epsilon$:
\begin{align*}
\hat{\mu}_i - \tau &\geq \mu_i(T) - \tau - \frac{\Delta_i}{2} \\
&\geq \mu_i - \tau - \frac{1}{2}\left(\mu_i - \tau + \epsilon\right) \\
&\geq \frac{1}{2}(\mu_i - \tau - \epsilon) \\
&\geq 0
\end{align*}

Similar argument holds for arms with $\mu_i <\tau - \epsilon$. We conclude that with high probability $1 - \delta$, all arms are correctly identified by the algorithm. As for the remaining event with probability $\delta$, we do not know whether the loss is $0$ or not. This yields
$$
\mathbb{E}[\mathcal{L}(T)] \leq \delta
$$
\end{proof}

\section{Proof of Lemma 3}
\label{app:lemma3}

\begin{proof}
    When $x_i = 0$, the above inequality is trivial. Assume $x_i \neq 0$:
    \begin{align*}
    \lVert x_i\rVert^2_{V_t^{-1}} &\leq \lVert x_i\rVert^2_{(I + T_i(t)x_ix_i^T)^{-1}} && \text{By lemma 1} \\
    &=  x_{i}^T   (I + T_i(t)x_ix_i^T)^{-1} x_{i} \\
    &= \lVert x_i\rVert^2 - \frac{x_i^T T_i(t)x_ix_i^T x_i}{1 + T_i(t)\lVert x_i\rVert^2}&& \text{By sherman-morrison formula} \\
    &= \lVert x_i\rVert^2 - \frac{T_i(t) \lVert x_i\rVert^4}{1 + T_i(t)\lVert x_i\rVert^2} \\
    &= \frac{\lVert x_i\rVert^2 + T_i(t)\lVert x_i\rVert^4 - T_i(t)\lVert x_i\rVert^4}{1 + T_i(t)\lVert x_i\rVert^2} \\
    &= \frac{\lVert x_i\rVert^2}{1 + T_i(t)\lVert x_i\rVert^2} \\
    &\leq \frac{\lVert x_i\rVert^2}{T_i(t)\lVert x_i\rVert^2} = \frac{1}{T_i(t)}\\
    \end{align*}
\end{proof}

\end{document}